\documentclass[]{article}
\usepackage{proceed2e} 
\usepackage[round]{natbib}
\usepackage{amsmath}
\usepackage{amssymb}
\usepackage{amsfonts}
\usepackage{bm}
\usepackage{enumitem}
\usepackage{url}
\usepackage{hyperref}
\usepackage{times}
\usepackage{pstool}
\usepackage{subcaption}

\usepackage{algorithm}
\usepackage[noend]{algpseudocode}

\usepackage{amsthm}

\newtheorem{theorem}{Theorem}[section]
\newtheorem{corollary}{Corollary}[theorem]
\newtheorem{lemma}[theorem]{Lemma}
\newtheorem{assumption}{Assumption}

\newtheorem{remark}{Remark}

\usepackage{float}
\usepackage{times}
\usepackage{bm}

\usepackage{scalerel,stackengine}
\stackMath
\newcommand\reallywidehat[1]{%
\savestack{\tmpbox}{\stretchto{%
  \scaleto{%
    \scalerel*[\widthof{\ensuremath{#1}}]{\kern-.6pt\bigwedge\kern-.6pt}%
    {\rule[-\textheight/2]{1ex}{\textheight}}
  }{\textheight}%
}{0.5ex}}%
\stackon[1pt]{#1}{\tmpbox}%
}

\DeclareMathOperator{\Tr}{Tr}
\DeclareMathOperator{\diag}{diag}

\title{Embarrassingly parallel MCMC using deep invertible transformations} 

\author{ \textbf{Diego Mesquita} \quad \textbf{Paul Blomstedt} \quad \textbf{Samuel Kaski}   \\
Helsinki Institute for Information Technology HIIT, Department of Computer Science, Aalto University \\
\texttt{\{diego.mesquita, paul.blomstedt, samuel.kaski\}@aalto.fi}
}

\begin{document}

\maketitle

\begin{abstract}
While MCMC methods have become a main work-horse for Bayesian inference, scaling them to large distributed datasets is still a challenge.
Embarrassingly parallel MCMC strategies take a divide-and-conquer stance to achieve this by writing the target posterior as a product of subposteriors, running MCMC for each of them in parallel and subsequently combining the results.
The challenge then lies in devising efficient aggregation strategies. 
Current strategies  trade-off between approximation quality, and costs of communication and computation.
In this work, we introduce a novel method that addresses these issues simultaneously.
Our key insight is to introduce a deep invertible transformation to approximate each of the subposteriors. These approximations can be made accurate even for complex distributions and serve as intermediate representations, keeping the total
communication cost limited.
Moreover, they enable us to sample from the product of the subposteriors using an efficient and stable importance sampling scheme.
We demonstrate that the approach outperforms available state-of-the-art methods in a range of challenging scenarios, including high-dimensional and heterogeneous subposteriors.
\end{abstract}


%

\section{INTRODUCTION}


Markov Chain  Monte Carlo (MCMC) algorithms have cemented themselves as a cornerstone of practical Bayesian analysis.  Nonetheless, accommodating large distributed datasets is still a challenge. For this purpose, methods have been proposed to speed up inference either using  mini-batches \citep[e.g.][]{Ma2015, Quiroz2018} or exploiting parallel computing \citep[e.g.][]{Ahn2014, Johnson2013}, or combinations thereof. For a comprehensive review about scaling up Bayesian inference, we refer to \citet{Angelino+others:2016} and \citet{Robert2018}.

A particularly efficient class of parallel algorithms are embarrassingly parallel MCMC methods, which 
%
%
employ a divide-and-conquer strategy to obtain samples from the posterior
\begin{equation*}
    p(\theta | \mathcal{D}) \propto p(\theta) p(\mathcal{D} | \theta),
\end{equation*}
where $p(\theta)$ is a prior,  $p(\mathcal{D} | \theta)$ is a likelihood function and the data $\mathcal{D}$ are partitioned into $K$ disjoint subsets $\mathcal{D}_1, \ldots, \mathcal{D}_K$.
The general idea is to break the global inference into smaller tasks and combine their results, requiring coordination only in the final aggregation stage.  
More specifically, the target posterior is factorized as
\begin{equation}\label{eq:pde}
    p(\theta | \mathcal{D}) \propto \prod_{k=1}^{K} p(\theta)^{1/K} p(\mathcal{D}_k | \theta) ,
\end{equation}
and the right-hand-side factors, referred to as \emph{subposteriors}, are independently sampled from---in parallel---using an MCMC algorithm of choice. 
The results are then centralized in a coordinating server and aggregated. 
The core challenge lies in devising strategies which are both accurate and computationally convenient to combine subposterior samples. 

The seminal work of \citet{41849} approximates posterior samples as weighted averages of subposterior samples.
\citet{Neiswanger2014} proposed parametric, semi-parametric and non-parametric strategies, the two former being based on fitting kernel density estimators to the subposterior samples.
\citet{Wang2015} used random partition trees to learn a discrete approximation to the posterior. 
\citet{nemeth2018} fitted Gaussian process approximations to the log-subposteriors and took the product of their expected values.
%
%
Except for the parametric method, which imposes overly simplistic local approximations that generally result in poor approximations of the target posterior, all of the aforementioned approaches require the subposterior samples to be centralized, incurring extensive communication costs.
In fact, communication costs have been altogether ignored in the literature so far.
Furthermore, sampling from the approximate posterior can become difficult, requiring expensive additional MCMC steps to obtain samples from the combined posterior.



In this work, we propose a novel embarrassingly parallel MCMC strategy termed \emph{non-volume-preserving aggregation product} (NAP), which addresses the aforementioned issues while providing accurate posterior samples. 
Our work builds on the insight that subposteriors of arbitrary complexity can be mapped to densities of tractable form, making use of 
\emph{real non-volume preserving trasformations} (real NVP),
a recently developed class of 
neural-network based invertible transformations \citep{Dinh2017}.
This enables us to accurately evaluate the subposterior densities and sample from the combined posterior using importance sampling.
%
%
We prove that, under mild assumptions, our importance sampling scheme is  stable, i.e., estimates for a test function $h$ have finite variance. %

Experimental results show that NAP outperforms state-of-the art methods in several situations, including heterogeneous subposteriors and intricate-shaped, multi-modal or high-dimensional posteriors. 
Finally, the proposed strategy results in communication costs which are constant in the number of subposterior samples, which is an appealing feature when communication between machines holding data shards and the server is expensive or limited. 

The remainder of this work proceeds as follows. Section \ref{sec:meth} introduces our method, covering the required background on real NVP transformations.
Section \ref{sec:experiments} presents experimental results.  We conclude with a discussion on the results and possible unfoldings of this work in Section \ref{sec:discussion}.


\section{METHOD}
\label{sec:meth}

In this work, we employ real NVP transformations to approximate subposteriors using samples obtained from independent MCMC runs. In the following subsections, we 1) review the basics of real NVP transformations; 2) discuss how to combine them using importance sampling and 3) how to obtain samples from the approximate posterior using sampling/importance resampling.

\subsection{REAL NVP DENSITY ESTIMATION}

Real NVP \citep{Dinh2017} is a class of deep generative models in which a $D$-dimensional real-valued quantity of interest $x$ is modeled as a composition of bijective transformations from a base latent variable $z$, with known density function $p_{Z}$, \emph{i.e.}:
\begin{equation*}
    x = g_L \circ  g_{L-1} \circ \ldots \circ g_1(z) = g(z),
\end{equation*}
such that $g_l : \mathbb{R}^D \rightarrow \mathbb{R}^D$ for all $1 \leq l \leq L$. 
The density $p_X(x)$ is then obtained using the change-of-variable formula
\begin{equation}\label{eq:change_var}
    p_X(x) = p_{Z}\big(f(x)\big) \bigg| \det \frac{\partial f(x)}{\partial x^\top} \bigg|,
\end{equation}
where 
\[
f = f_1 \circ f_2 \circ \ldots \circ f_L = g_1^{-1} \circ g_2^{-1} \circ \ldots \circ g_L^{-1} = g^{-1}.
\]
To make (\ref{eq:change_var}) tractable, it is composed as follows. 
Let $\mathcal{I}_l \subset \{1, \ldots, D\}$ be a pre-defined proper subset of indices with cardinality $|\mathcal{I}_l|$, and denote its complement by $\overline{\mathcal{I}}_l$. Then, each transformation $v^\prime = f_l(v)$ is computed as:
\begin{align} \label{eq:f_l}
    &v^\prime_{\mathcal{I}_l} = v_{\mathcal{I}_l} \nonumber\\
    &v^\prime_{\overline{\mathcal{I}}_l} = v_{\overline{\mathcal{I}}_l} \odot \exp\{s_l(v_{{\mathcal{I}}_l})\} + t_l(v_{{\mathcal{I}}_l}),
\end{align}
where $\odot$ is an element-wise product. The functions $s_l, t_l : \mathbb{R}^{|\mathcal{I}_l|} \rightarrow \mathbb{R}^{|\overline{\mathcal{I}}_l|}$ are deep neural networks, which perform scale and translation, respectively. 
In particular, the Jacobian of $f_l$, has the form  
\begin{align*}
\frac{\partial v^\prime}{\partial v} = \begin{bmatrix} \mathbb{I}_{|\mathcal{I}_l|} & 0 \\  \frac{\partial v^\prime_{\overline{\mathcal{I}}_l}}{\partial v_{\mathcal{I}_l}} & \diag\big(\exp\{s_l(v_{\mathcal{I}_l})\}\big)
\end{bmatrix},
\end{align*}
which avoids explicit computation of the Jacobian of the functions $s_l$ and $t_l$. For observed data $(x_1,\ldots,x_N)$, the weights of the networks $s_l$ and $t_l$ that implicitly parameterize $p_X$ are estimated via maximum likelihood.  

Sampling from $p_X$ is inexpensive and resumes to sampling $z \sim p_{Z}$, and computing $x=g(z)$. Here, each $g_l$ is of the form
\begin{align} \label{eq:g_l}
    &v_{\mathcal{I}_l} = v^\prime_{\mathcal{I}_l} \nonumber\\
    &v_{\overline{\mathcal{I}}_l} 
     = ( v^{\prime}_{\overline{\mathcal{I}}_l} -  t_l(v^\prime_{{\mathcal{I}}_l}) ) \odot \exp\{ - s_l(v^\prime_{{\mathcal{I}}_l})\},
\end{align}
where (\ref{eq:g_l}) is obtained from (\ref{eq:f_l}) by straightforward inversion.

\subsection{COMBINING LOCAL INFERENCES}

Consider now a factorization of a target posterior density $p(\theta | \mathcal{D})$ into a product of $K$ subposteriors according to Equation~(\ref{eq:pde}). In embarrassingly parallel MCMC, each worker runs MCMC independently on its respective subposterior,
\begin{equation*}
    p_k(\theta) := \frac{1}{Z_k} p(\theta)^{1/K} p(\mathcal{D}_k | \theta),
\end{equation*}
to obtain a set of draws $\{\theta_{s}^{(k)}\}_{s=1}^S$ from $p_k(\theta)$. 
The goal is then to produce draws from an approximate target posterior $\reallywidehat{p}(\theta) \approx  p(\theta | \mathcal{D})$, using the $K$ sets of subposterior samples as input. 
This requires estimating the densities $p_k(\theta)$ from the subposterior samples, and sampling from the distribution induced by the product of approximations 
\begin{equation} \label{eq:approx_pde}
\reallywidehat{p}(\theta)  \propto  \reallywidehat{p}_1(\theta) \reallywidehat{p}_2(\theta) \ldots \reallywidehat{p}_K(\theta),
\end{equation}
typically resulting in a trade-off between accuracy and computational efficiency.

In this work, we make use of the fact that bijective transformations using real NVP offers both accurate density estimation and computationally efficient sampling for arbitrarily complex distributions. To this end, we first fit a separate real NVP network to estimate each $p_k$ as $\reallywidehat{p}_k$. 
The networks are then sent to a server that approximates the global posterior as in Equation~(\ref{eq:approx_pde}).

In a typical scenario, one would ultimately be interested in using $\reallywidehat{p}$ to compute the expectation of some function $h : \mathbb{R}^D \rightarrow \mathbb{R}$, such as a predictive density or a utility function. 
In our case, straightforward importance sampling can be used to weight samples drawn from any of the subposteriors. Thus, given a set of $T$ samples drawn from any $\reallywidehat{p}_k$, we obtain the estimate:
\begin{equation*}
  \overline{h}(\theta) = \sum_{t=1}^{T}  w_t h(\theta_t),
\end{equation*}
where the importance weights $w_1, \ldots, w_T$ are normalized to sum to one, and given by 
\begin{equation}\label{eq:weight}
w_t \propto  \frac{\prod_{k^\prime=1}^{K} \reallywidehat{p}_{k^\prime}(\theta_t)}{\reallywidehat{p}_{k}(\theta_t)}.
\end{equation}
This strategy capitalizes on the key properties of real NVP transformations---ease of evaluation and sampling---and avoids the burden of running still more MCMC chains to sample from the aggregated posterior $\reallywidehat{p}(\theta)$, which might be a complicated target due to the underlying neural networks.

While importance sampling estimates can be unreliable if their variance is very high or infinite, we can provide guarantees that $\overline{h}(\theta)$ has finite variance. 
\citet{Geweke89} showed that, for a broad class of test functions, it suffices to prove that ${\prod_{k^\prime=1}^{K} \reallywidehat{p}_{k^\prime}(\theta)}/{\reallywidehat{p}_{k}(\theta)} \leq M$ $\forall \theta$, i.e., the importance weights are bounded. 
We first note that the denominator of the weight in Equation~(\ref{eq:weight}) is included as a factor in the numerator, so that $\reallywidehat{p}_{k}(\theta)=0 \Rightarrow \prod_{k^\prime=1}^{K} \reallywidehat{p}_{k^\prime}(\theta) = 0$. The remaining thing to check is that $\reallywidehat{p}_k(\theta)$ is bounded for all $k$ and all $\theta$.

We begin by making the following assumption on the structure of the neural networks which define the real NVP transformations.
\begin{assumption}
    The neural networks $s_1^{(k)}$, \ldots, $s_L^{(k)}$ associated with the real NVP estimate $\reallywidehat{p}_k$ are equipped with bounded activation functions in their individual output layers. \label{assumption:bounded_networks}
\end{assumption}

\begin{remark}
Note that Assumption \ref{assumption:bounded_networks} is satisfied, for example, when the activation functions in the last layer of the scale networks are the hyperbolic tangent or the logistic function. 
\end{remark}
We place no further assumption on the structure of the remaining layers of $s_1^{(k)}, \ldots, s_L^{(k)}$ or in the overall structure of the translation networks $t_1^{(k)}, \ldots, t_L^{(k)}$.

With the additional condition that we choose an appropriate density for the base variable of the NVP network, we can prove that $\reallywidehat{p}_k$ itself is bounded. 

\begin{lemma}
Given a bounded base density $p_{Z}$, the distribution resulting from $L$ transformations is bounded.  \label{lemma:bounded_nvp}
\end{lemma}
\begin{proof}
As $p_{Z}$ is bounded, there exists some constant $M > 0$ such that
    \begin{equation*}
        p_{Z}(z) \leq M \quad \forall z \in \mathbb{R}^D .
    \end{equation*}
     Let $v_1 = g_1(z)$. Applying the change-of-variable formula we  get
    \begin{align*}
       \log p_{v_1}(v_1) &= \log  p_{Z}(z) + \log\bigg| \det \frac{\partial z}{\partial v_1^\top} \bigg|\\
       &= \log  p_{Z}(z) + \bigg| \Tr\, \diag\big(s_l(v_{\mathcal{I}_l})\big)  \bigg| 
    \end{align*}    
 Let $\mathcal{B}_z > 0$ be the constant bounding $p_{Z}$.
 Using Assumption \ref{assumption:bounded_networks}, since all of the outputs of the neural networks $s_l$  are bounded, their sum is bounded by some $\mathcal{B}_s > 0$.  
  Then, it follows:
 \begin{equation*}
     \log p_{v_1}(v_1) \leq \mathcal{B}_z + \mathcal{B}_s,
 \end{equation*}
 i.e.,  $p_{v_1}$ is bounded.
 Repeating the argument we get a proof by induction on the number of transformations $L$.

\end{proof}

As a direct application of Lemma \ref{assumption:bounded_networks}, we get the desired bound the importance weights. 
\begin{theorem}
    For any $1 \leq k \leq K$, there exists $M>0$ such that for all $\theta \in \mathbb{R}^D$, ${\prod_{k^\prime} \reallywidehat{p}_{k^\prime}(\theta) }/{\reallywidehat{p}_{k}(\theta)} \leq M$.
\end{theorem}
\begin{proof}
    Using Lemma \ref{lemma:bounded_nvp},  let $\mathcal{U}_{k^\prime}$ be the upper bound for $\reallywidehat{p}_{k^\prime}$ and let $M = \prod_{k^\prime \neq k} \mathcal{U}_{k^\prime}$, from which the statement follows.
\end{proof}

This provides the sufficient conditions underlined by \citet{Geweke89}, so that we achieve the following result regarding the overall stability of the importance sampling estimates.
\begin{corollary}
Suppose $\theta_{1} ,\ldots, \theta_{T}$ are samples from $\reallywidehat{p}_k$ for some $1 \leq k \leq K$. Let $w_t \propto {\prod_{k^\prime} \reallywidehat{p}_{k^\prime}(\theta_t) }/{\reallywidehat{p}_{k}(\theta_t)}$, $\sum_t w_t = 1$ and $\mathrm{Var}_{\reallywidehat{p}} [h] < \infty$. Then, the importance sampling estimate $\overline{h}(\theta) = \sum_{t=1}^{T} w_t h(\theta_t)$ has finite variance.
\end{corollary}

\subsection{SAMPLING FROM THE APPROXIMATE POSTERIOR}

We can also use the samples $\theta_1, \ldots, \theta_T$ from $\reallywidehat{p}_k$ and their associated importance weights $w_1, \ldots, w_T$ to obtain approximate samples $\theta_1^\star,  \ldots, \theta_R^\star$ from $\reallywidehat{p}$ using sampling/importance resampling (SIR). With this, $\mathbb{E}_{\reallywidehat{p}}[h(\theta)]$ can be directly estimated as a Monte Carlo integral over the new samples. This procedure easily is done by choosing $\theta_r^\star = \theta_t$ with probability proportional to $w_t$. The required steps are detailed in Algorithm \ref{algo:sir}.

\begin{algorithm}
\caption{NAP-SIR}
\label{algo:sir}
\begin{algorithmic}[1]
    \Require Subposterior approximations $\reallywidehat{p}_1, \ldots, \reallywidehat{p}_K$, number of candidate samples $T$, final number of samples $R$, chosen subposterior index $k \in \{1, \ldots,  K\}$.
    \Ensure $R$ samples $\theta_1^\star, \ldots, \theta_R^\star$ from  $\reallywidehat{p}$. \vspace{0.1in}
    \For{$t = 1,\ldots, T$}
        \State Sample $\theta_t \sim \reallywidehat{p}_k$
        \State $w_t \leftarrow  {\prod_{k^\prime} \reallywidehat{p}_{k^\prime}(\theta_t)}/{\reallywidehat{p}_{k}(\theta_t)}$
    \EndFor
    \State $c \leftarrow \sum_t w_t$
    \State  $\bm{w} \leftarrow (w_1/c, \ldots, w_T/c)$
       \For{$r = 1, \ldots, R$} 
        \State  Draw $t$ from $\mathrm{Categorical}(\bm{w})$
        \State  $\theta_r^\star \leftarrow \theta_t$ 
    \EndFor
\end{algorithmic}
\end{algorithm}

Note that Algorithm~\ref{algo:sir} provides, for any single $k$, a valid sampler for the approximate posterior $\reallywidehat{p}$. However, in practice it is beneficial to apply the algorithm for many or all $k$ to provide better exploration of the parameter space.

\subsection{TIME COMPLEXITY}

We now analyze the time complexity of the proposed method with respect to the number of subposteriors $K$, the number of samples $S$ drawn from each of the subposteriors $p_1, \ldots, p_K$, and the number of samples $R$ which we wish to obtain from the aggregated posterior.

Obtaining $R$ samples from the approximate posterior using NAP consists of a single pass of the two following steps:
\begin{description}
    \item[\textbf{Step 1.}] In parallel, for $k=1,\ldots,K$, fit a real NVP transformation to the samples drawn from the $k$th subposterior at worker $k$.
    \item[\textbf{Step 2.}] Gather the subposterior approximations. Choose a $ k \in \{1, \ldots, K\}$, choose $T \geq R$ and use  Algorithm \ref{algo:sir} to draw $R$ samples from $\reallywidehat{p}$. 
\end{description}

Step 1 involves the usual costs of learning real NVP networks, which can be done using gradient-based methods, such as ADAM \citep{ADAM}.
%
Assuming the number of layers and weights per layer in each network is fixed, evaluating $\reallywidehat{p}(\theta) = \prod_k \reallywidehat{p}_k(\theta)$ takes linear time in $D K$.  Further taking $T = \lceil c R \rceil$ for some constant $c\geq1$, we conclude Step 2 can be executed in $\mathcal{O}(R D K + R^2)$. 

\subsection{COMMUNICATION COSTS}

It is important to note that typically $S \gg |\mathcal{D}_k|$, i.e., a worker ouputs a much larger number of  subposterior samples than the size of the data subset $\mathcal{D}_k$ it processes. 
Even if the dataset is split among workers to improve computational efficiency through parallel inference, 
sending subposterior samples back to the server for aggregation can amount to considerable communication costs.
Therefore, we also examine communication cost of NAP, and contrast it to currently available methods.

The communication cost of the proposed NAP amounts to $\mathcal{O}(K D)$, corresponding  to the cost of communicating the NVP networks to the server, which does not depend on $S$. 
On the other hand, current methods have their accuracy intrinsically tied to the number of  subposterior samples communicated to the server, resulting in $\mathcal{O}(S K D)$ communication costs. 
For example, the partition tree based method of  \citet{Wang2015} requires recursive pair-wise aggregation of subposteriors, which calls for centralization of subposterior samples.
%
The non-parametric and semi-parametric methods proposed by \citet{Neiswanger2014} require computing kernel-density estimates defined on each subposterior individually and centralizing them to subsequently execute an MCMC step to sample from their product. 
Similar costs are implied by the strategy of \citet{nemeth2018}, which fits Gaussian process approximations and centralizes them to use MCMC to sample from the product of the expected values of their exponentiated predictives.

In other words, with NAP, subposteriors can be made arbitrarily accurate by drawing more subposterior samples (as long as local resources allow) with no additional effect on the cost of communicating the networks to the server.

\section{EXPERIMENTAL RESULTS}
\label{sec:experiments}
We evaluated the performance of the proposed method in four different experiments, comparing it against several aggregation methods\footnote{We have used the implementations available at \\ \href{https://github.com/richardkwo/random-tree-parallel-MCMC}{https://github.com/richardkwo/random-tree-parallel-MCMC}}: 
\begin{itemize}
    \item \textbf{Parametric (PARAM):} approximates the posterior as a product of multivariate normal densities fitted to each subposterior \citep{Neiswanger2014}.
    \item \textbf{Non-parametric (NP)}: uses  kernel density estimates to approximate the subposteriors, takes their product and samples from it using Gibbs sampling \citep{Neiswanger2014}. 
    \item \textbf{Semi-parametric (SP)}: a hybrid between the two former approaches \citep{Neiswanger2014}.
    \item \textbf{Consensus (CON)}: takes weighted averages of subposterior samples to obtain approximate samples from the target posterior \citep{41849}.
    \item \textbf{Parallel aggregation using random partition trees (PART)}: uses partition trees to fit hyper-histograms to the target posterior using subposterior samples \citep{Wang2015}.
\end{itemize}

In the first experiment, we target a uni-modal distribution of an intricate shape.
In the second, we approximate a bi-variate multi-modal distribution.
In the third, we evaluate the performance of our method when approximating logistic regression posteriors in high dimensions.
Finally, in the last one we analyze the performance of our method when there is a clear discrepancy among the subposteriors being merged.

All MCMC simulations were carried out using the python interface of the Stan probabilistic programming language\citep{Carpenter2017}, which implements the no-U-turn sampler. For each subposterior, we draw 4000 samples using 16 chains with an equal number of samples as warm-up.  The same holds for the target (ground truth) posterior, computed on centralized data.
The real NVP networks were implemented with PyTorch\footnote{\href{https://pytorch.org}{https://pytorch.org}} using three transformations ($L=3$) and Gaussian spherical base densities. The scale and translation networks were all implemented as multi-layer perceptrons with two hidden-layers comprising $256$ nodes each. The layers of these networks were all equipped with rectified linear units except for the the last layer of each scale network, which was equipped with the hyperbolic tangent activation function.
The network parameters were optimized using ADAM \citep{ADAM} over 1000 iterations with learning rate $10^{-4}$.

\subsection{WARPED GAUSSIAN}

\begin{figure*}[htb!]
    \centering
    \begin{subfigure}{0.33\textwidth}
    \includegraphics[width=0.9\textwidth]{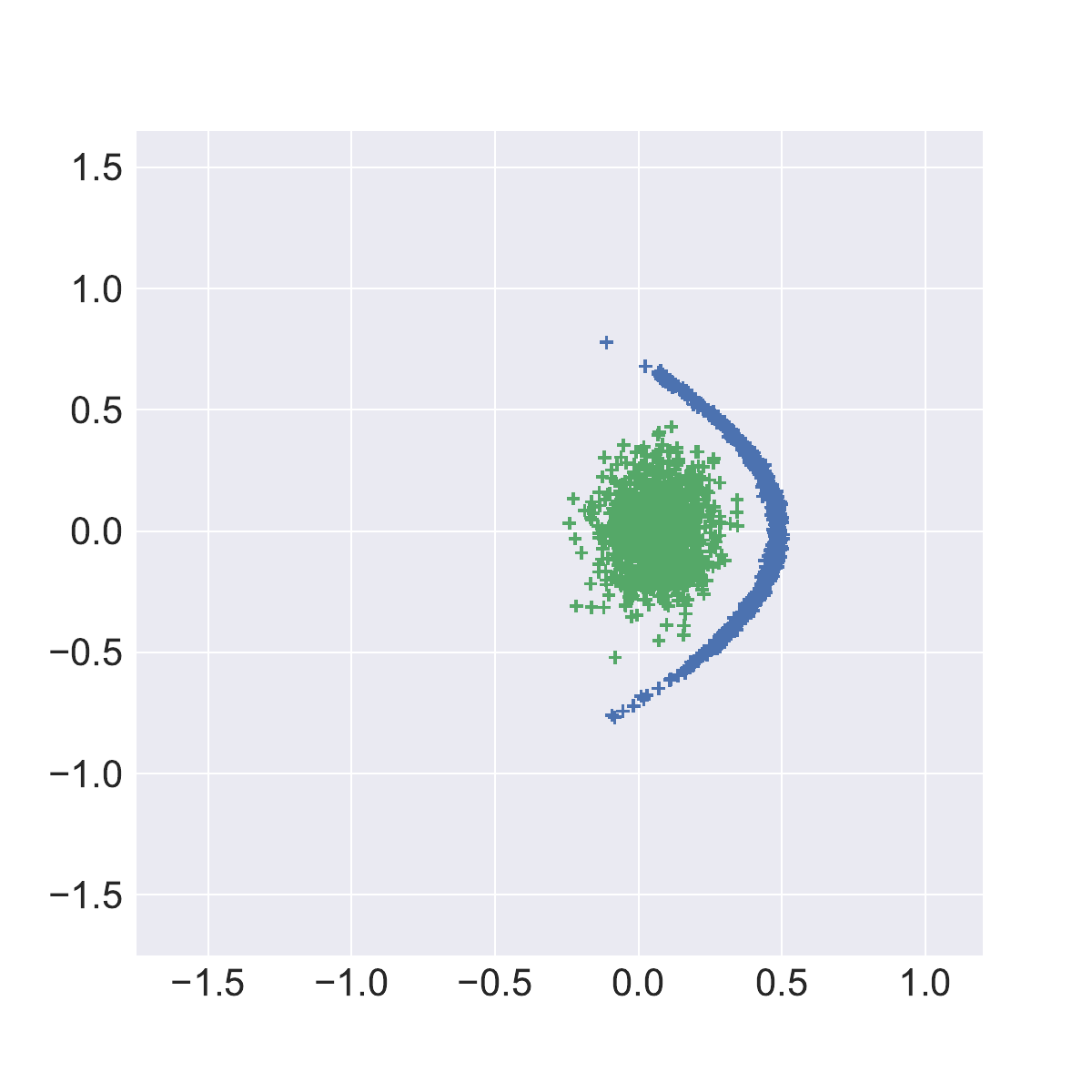}
    \caption{Parametric}
    \end{subfigure}
    \begin{subfigure}{0.33\textwidth}
    \includegraphics[width=0.9\textwidth]{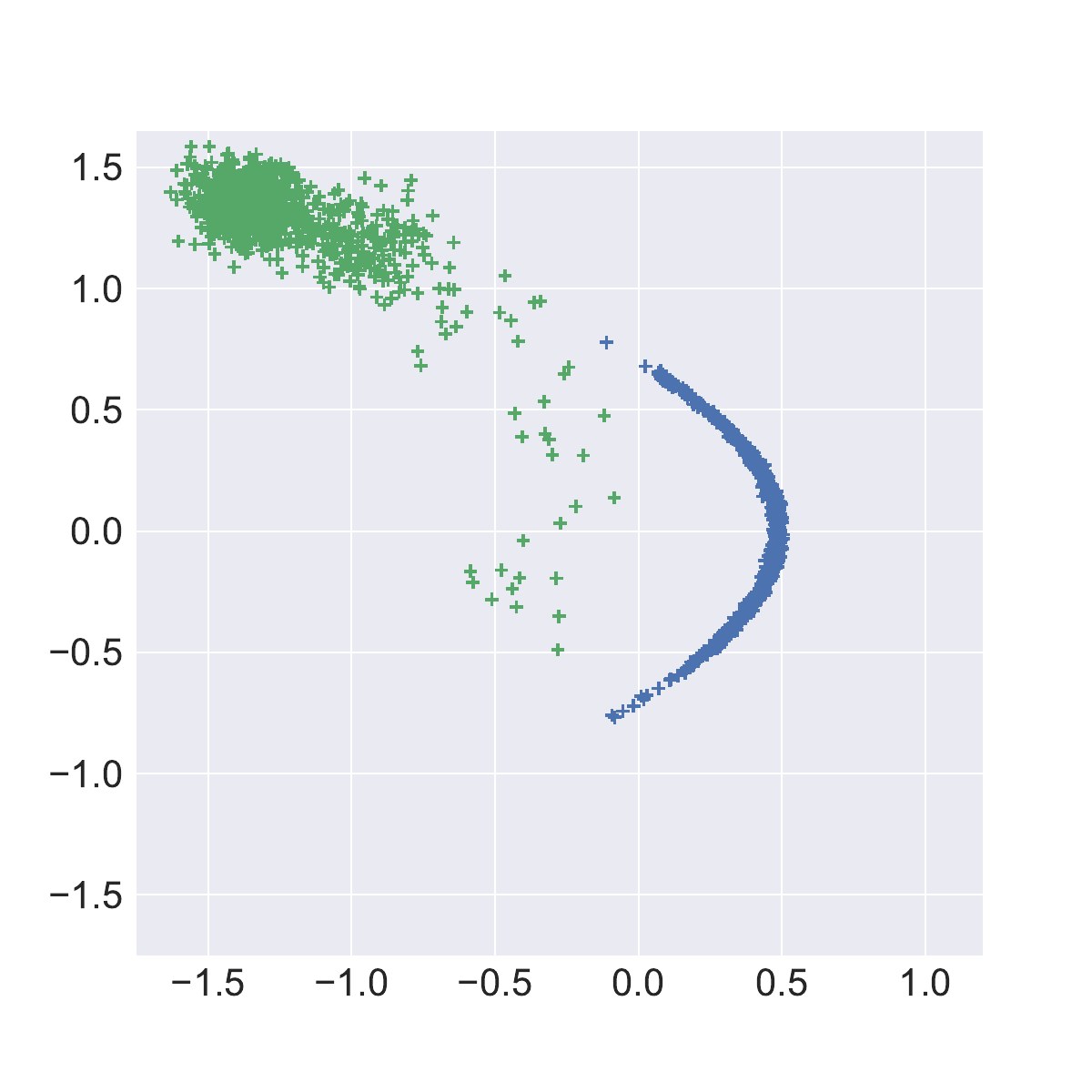}
    \caption{Semi-parametric}
    \end{subfigure}
    \begin{subfigure}{0.33\textwidth}
    \includegraphics[width=0.9\textwidth]{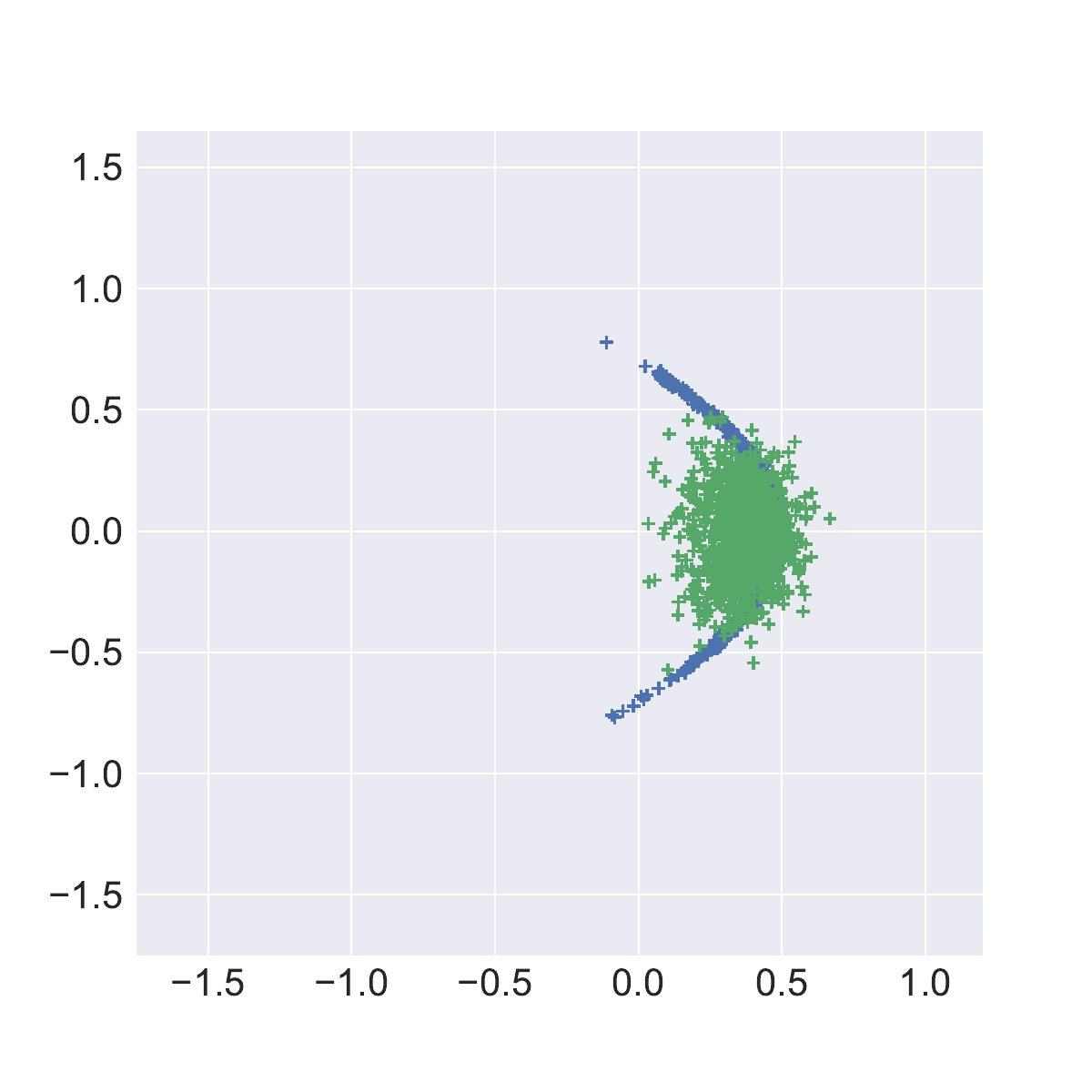}
    \caption{Non-parametric}
    \end{subfigure} \\
    \begin{subfigure}{0.33\textwidth}
    \includegraphics[width=0.9\textwidth]{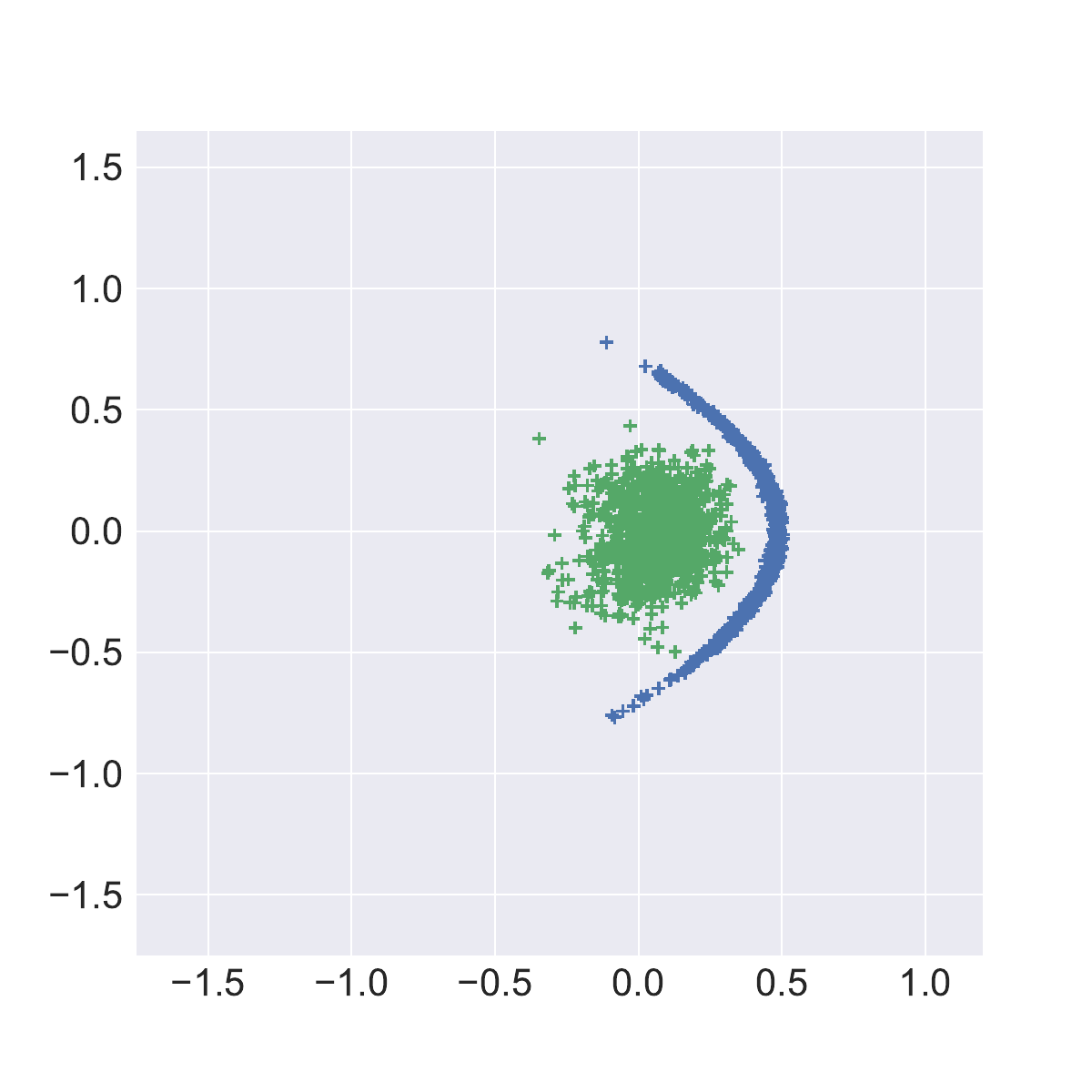}
    \caption{Consensus}
    \end{subfigure}
    \begin{subfigure}{0.33\textwidth}
    \includegraphics[width=0.9\textwidth]{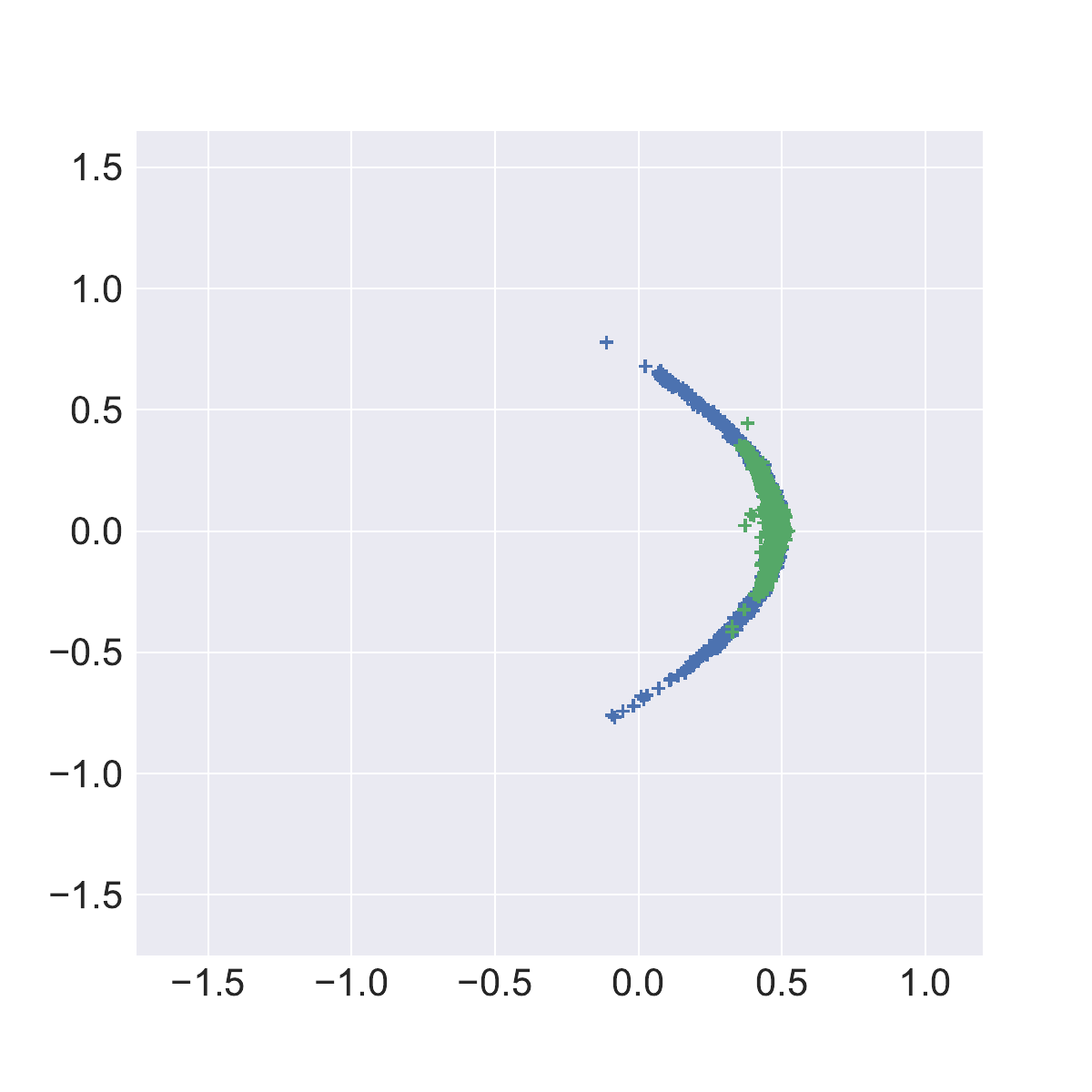}
    \caption{PART}
    \end{subfigure}
    \begin{subfigure}{0.33\textwidth}
    \includegraphics[width=0.9\textwidth]{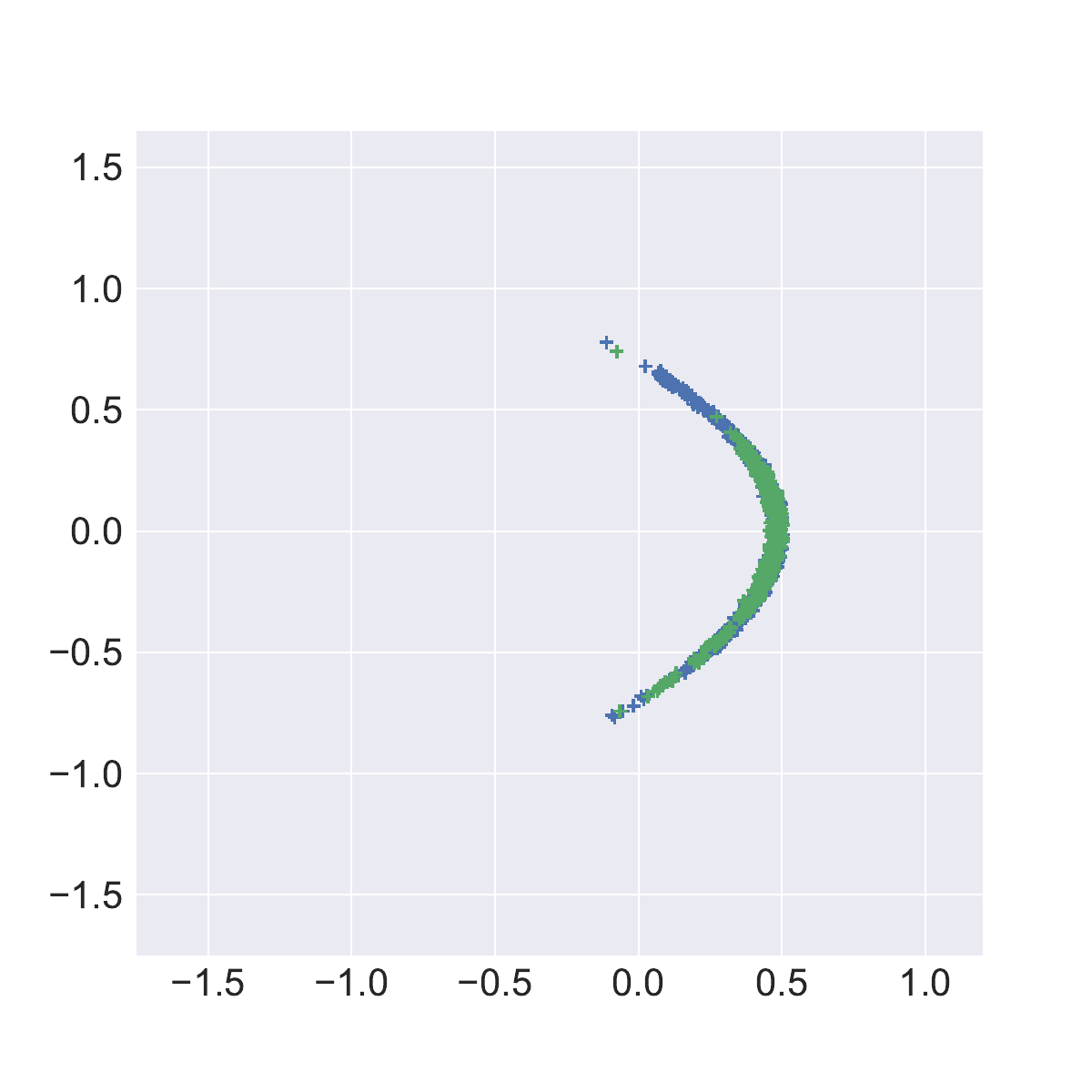}
    \caption{NAP}
    \end{subfigure}
    \caption{MCMC samples for the warped Gaussian model obtained on the centralized dataset (ground truth), in blue, against samples from posterior approximations using different embarassingly parallel MCMC methods, in green.}
    \label{fig:warped}
\end{figure*}

We first consider  inference in a warped Gaussian model which exhibits a banana-shaped posterior and is described by the generative model:
\begin{equation*}
y \sim \mathcal{N}(\mu_1 + \mu_2^2, \sigma^2),
\end{equation*}
where the true values of the parameters $\mu_1$ and $\mu_2$ are $0.5$ and $0$, respectively. The variance $\sigma^2$ is set to $2$ and treated as a known constant. We draw $10000$ observations from the model and distribute them in $K=10$ disjoint sets. Gaussian priors with zero mean and variance 25 were placed both on $\mu_1$ and $\mu_0$.

For NAP, we used Algorithm \ref{algo:sir} to draw $4000$ samples from the approximate posterior, using $16000$ samples drawn from the individual subposterior approximation. To avoid possible underflow from normalizing a large number of importance weights, we do this in $K$ installments, in each of which a suposterior approximation is used as a proposal. The same number of samples was drawn using each of the competing methods.

Figure \ref{fig:warped} shows\footnote{The experiment was repeated with multiple random seeds, yielding similar results.} the samples from the approximate posterior obtained with different aggregation methods, plotted against the posterior obtained using the entire sample set. 
Of all the methods, only NAP and PART were flexible enough to mimic the banana shape of the posterior. 
PART, however, is overly concentrated when compared to the ground truth, while NAP more faithfully spreads the mass of the distribution.

\subsection{MIXTURE OF BETAS}

\begin{figure*}[htb!]
    \centering
    \begin{subfigure}{0.33\textwidth}
    \includegraphics[width=0.9\textwidth]{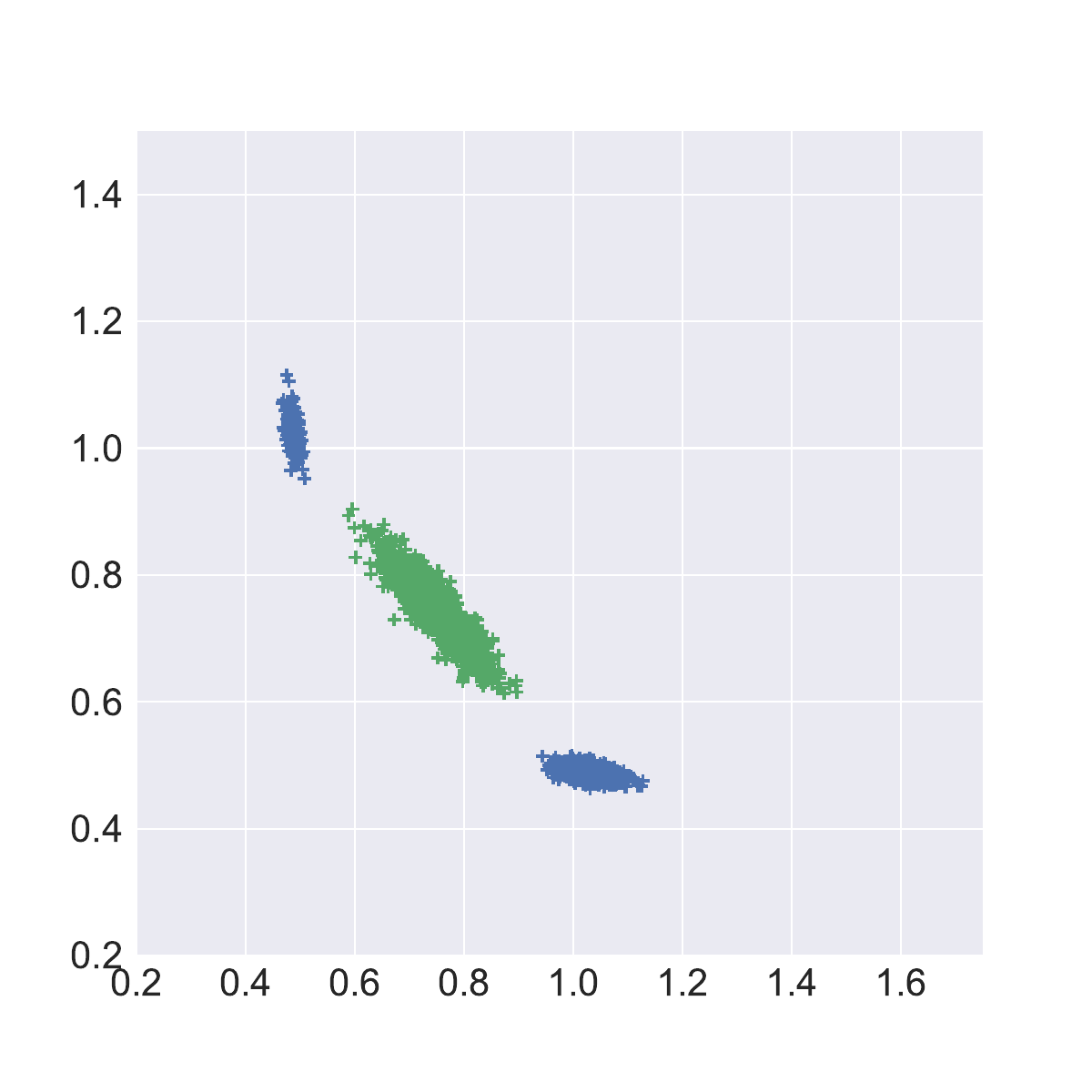}
    \caption{Parametric}
    \end{subfigure}
    \begin{subfigure}{0.33\textwidth}
    \includegraphics[width=0.9\textwidth]{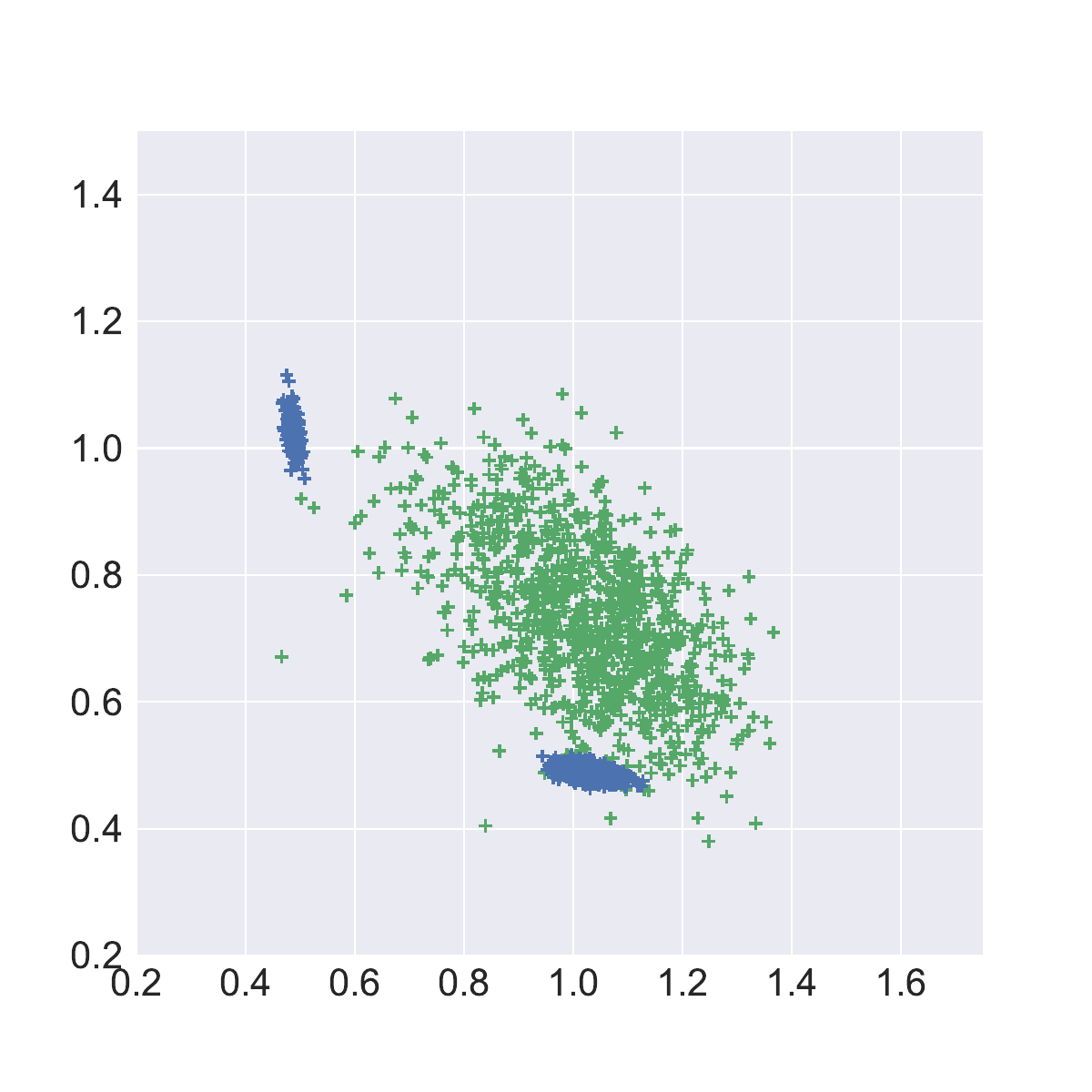}
    \caption{Semi-parametric}
    \end{subfigure}
    \begin{subfigure}{0.33\textwidth}
    \includegraphics[width=0.9\textwidth]{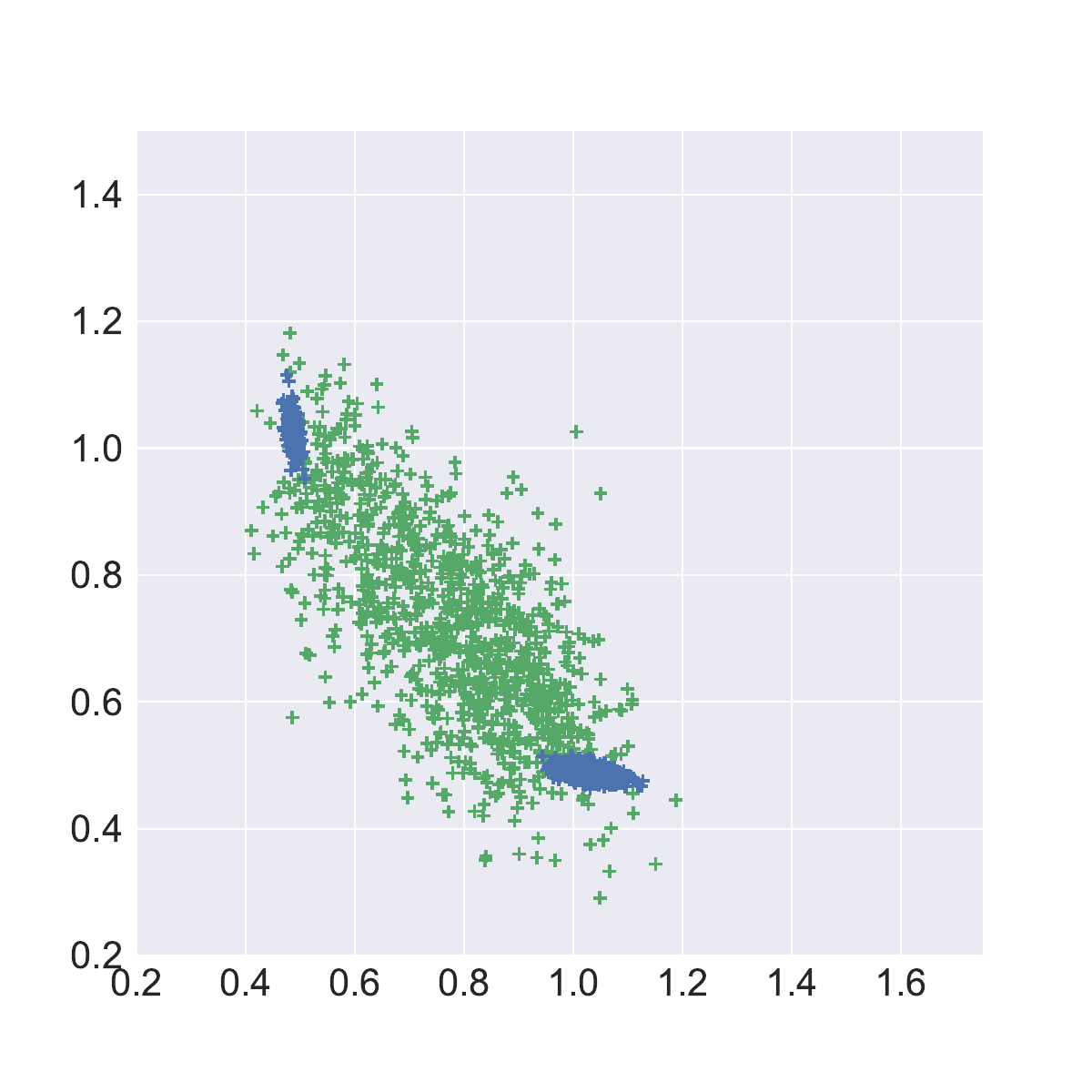}
    \caption{Non-parametric}
    \end{subfigure} \\
    \begin{subfigure}{0.33\textwidth}
    \includegraphics[width=0.9\textwidth]{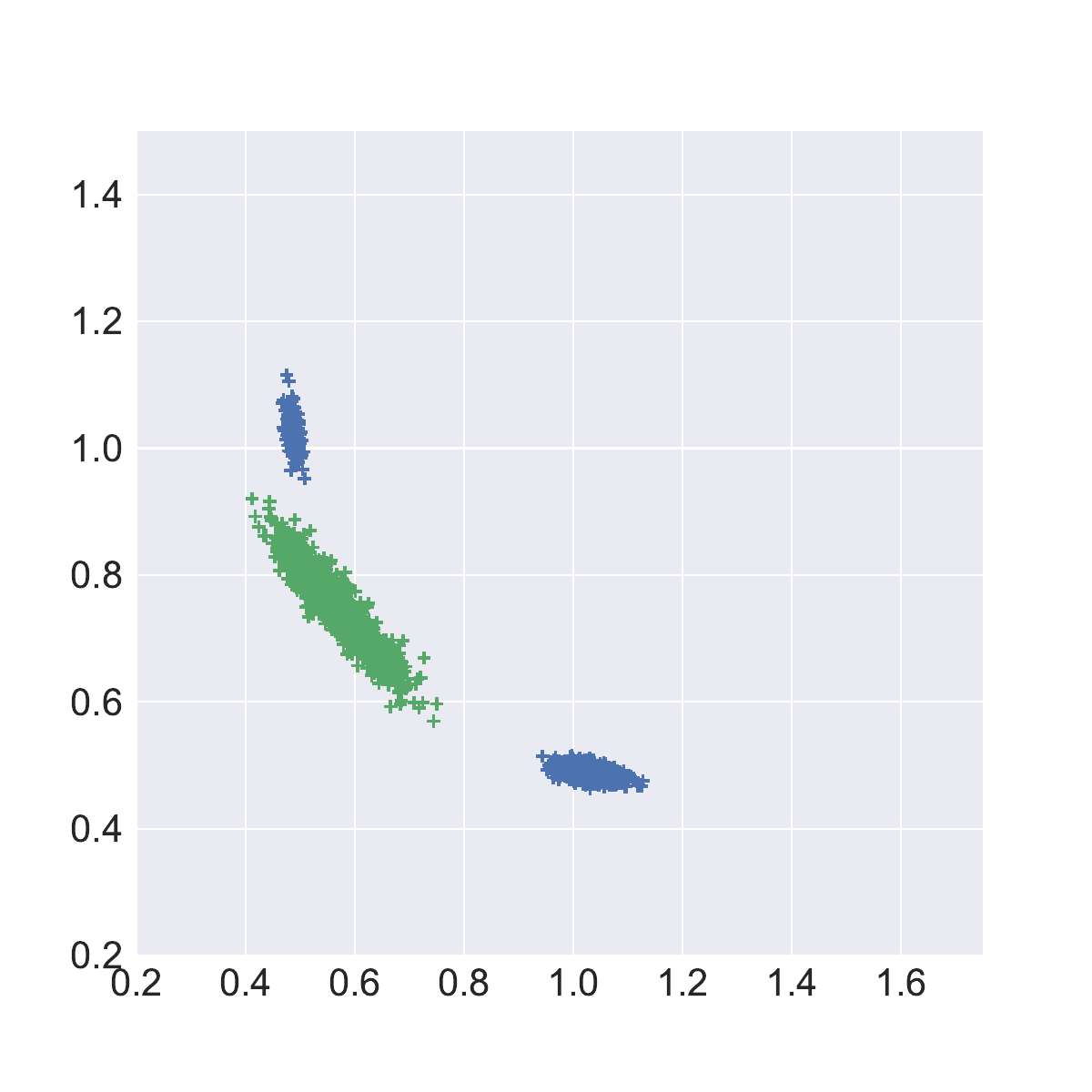}
    \caption{Consensus}
    \end{subfigure}
    \begin{subfigure}{0.33\textwidth}
    \includegraphics[width=0.9\textwidth]{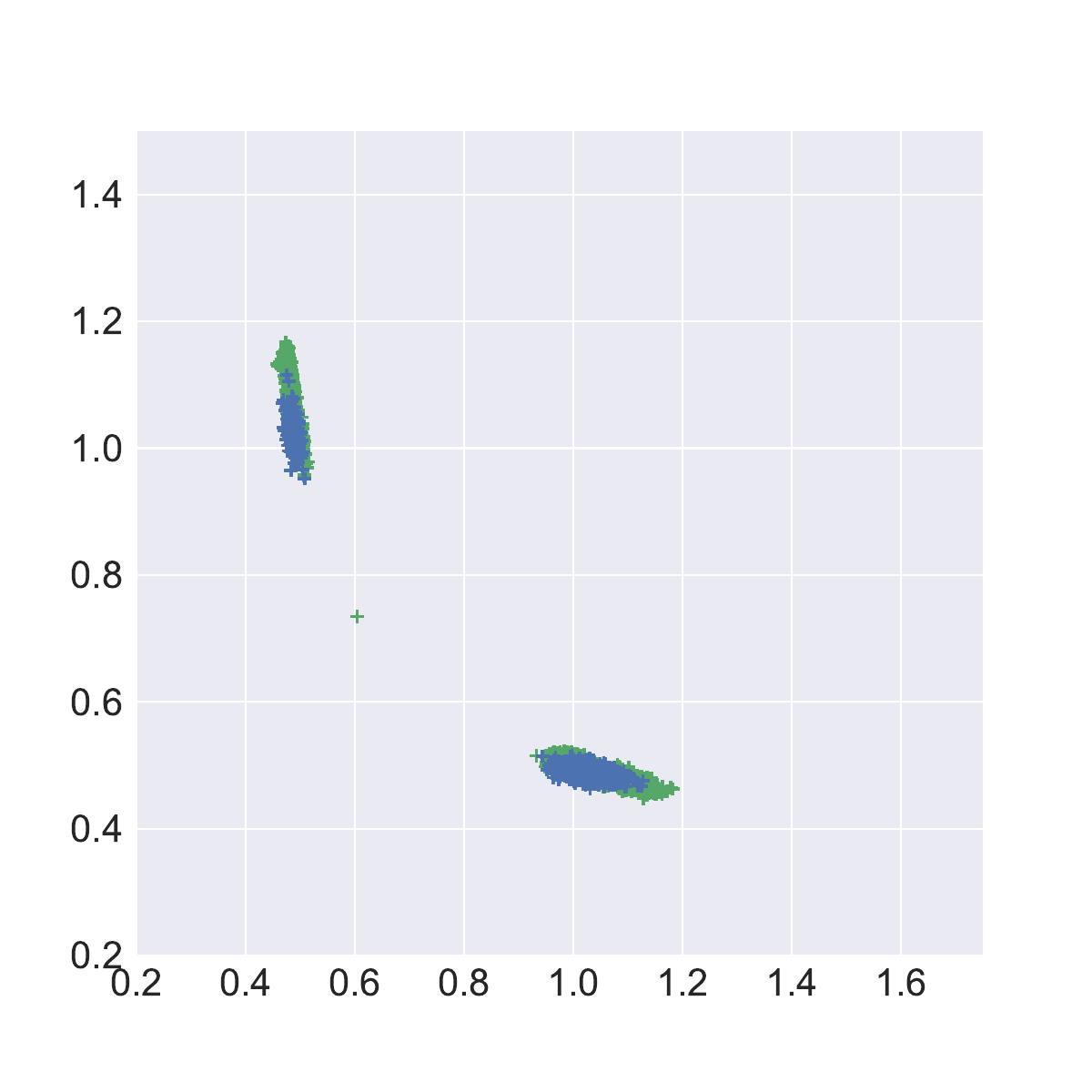}
    \caption{PART}
    \end{subfigure}
    \begin{subfigure}{0.33\textwidth}
    \includegraphics[width=0.9\textwidth]{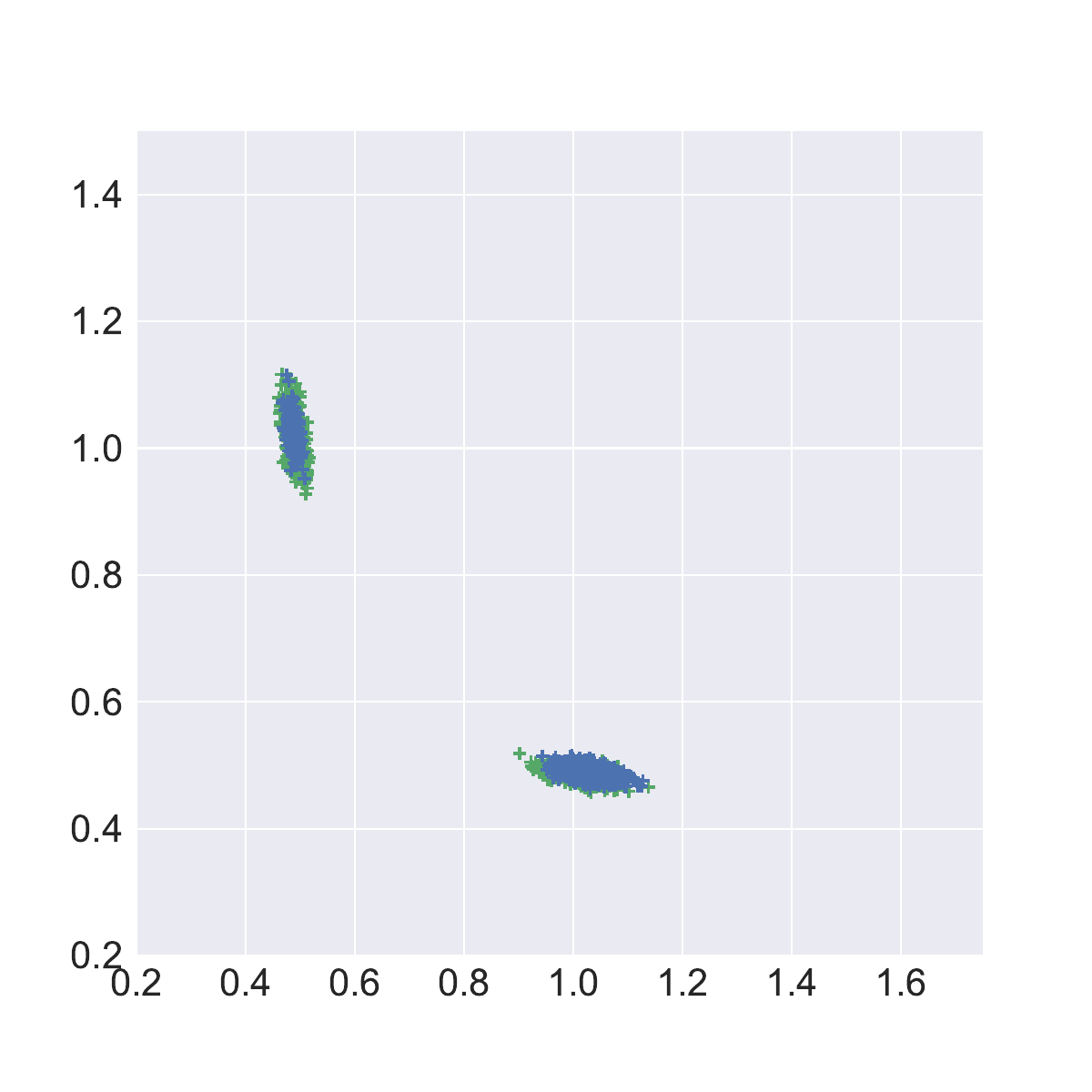}
    \caption{NAP}
    \end{subfigure}
    \caption{MCMC samples for the mixture of Betas model obtained on the centralized dataset (ground truth), in blue, against samples from posterior approximations using different embarassingly parallel MCMC methods, in green.}
    \label{fig:mixture}
\end{figure*}

We now consider performing inference in the shape parameters $\alpha_1$ and $\alpha_2$ of the following two-component mixture of Betas:
\begin{equation*}
    p(y | \alpha_1, \alpha_2) = \frac{1}{2} \mathrm{Beta}(y | \alpha_1, \beta_1) +  \frac{1}{2} \mathrm{Beta}(y | \alpha_2, \beta_2), 
\end{equation*}
where the true values of the parameters of interest are $\alpha_1=0.5$ and $\alpha_2=1.0$. Furthermore, $\beta_1 = \beta_2 = 1$ are known constants, which makes the model clearly bi-modal.
Independent Gamma priors with shape $0.5$ and scale $1.0$ were placed on $\alpha_1$ and $\alpha_2$.

As before, we drew $10000$ observations from the model and distributed them in $K=10$ disjoint sets for parallel inference. 
Samples from the approximate posterior for each aggregation algorithm were drawn in the same fashion as in the previous experiment.

Figure \ref{fig:mixture} shows\footnote{The experiment was repeated with multiple random seeds, yielding similar results.} the samples from the approximate posteriors obtained with each aggregation method plotted against the target posterior, obtained using the entire sample set. The proposed method and PART clearly are the only ones that capture the multi-modality of the posterior. NAP, however, presented a better fit to the true posterior while PART placed more mass in low-density regions. 

\subsection{BAYESIAN LOGISTIC REGRESSION}

We now explore how our method behaves in higher dimensions in comparison to its alternatives. For this purpose we consider inference on the simple logistic regression model with likelihood
\begin{align*}
    y_i \sim \mathrm{Bernoulli}\big(  \sigma(\theta_{1:p} \cdot x_i + \theta_0) \big) \quad \forall 1 \leq i \leq N ,
\end{align*}
where $\cdot$ denotes the dot product, $\sigma(t) = (1 + e^{-t})^{-1}$ is the logistic function, and $\theta_0, \ldots, \theta_p$ receive independent $\mathcal{N}(0, \sqrt{5})$ priors. The true value $\theta_0^{\prime}$ of $\theta_0$ is held at $-3$ and the remaining  $\theta_{1}^\prime, \ldots, \theta_p^\prime$ are independently drawn from a normal distribution with zero-mean and variance $0.25$.   

To generate a sample pair $(x_i, y_i)$, we first draw $x_i$ from $\mathcal{N}(\bm{0}, \Sigma)$, where the covariance matrix $\Sigma$ is such that
\[
    \Sigma_{i,j} = 0.9^{|i - j|} \quad \forall 1 \leq i, j \leq p.
\]
Then, $y_i$ is computed by rounding $\sigma(\theta_{1:p}^\prime \cdot x_i + \theta_0^\prime)$ to one if it is at least $0.5$, and to zero otherwise.


For each value of $p \in \{ 25, 50, 100 \}$, we draw $N=10000$ sample pairs using the scheme described above and distribute them in $K=50$ disjoint sets for parallel inference. 
As in previous experiments, we use NAP and its counterparts to merge the subposteriors and draw $4000$ samples from the approximate posterior.


Table \ref{tab:log_res} presents the results for each of the aggregation methods in terms of the following performance measures:
\begin{itemize}
    \item \textbf{Root mean squared error} (RMSE) between the mean $\overline{\theta}$ of the approximate posterior samples $\{\theta^\star_r\}_{r=1}^R$ and the mean $\overline{\theta}^\prime$ of samples $\{\theta_r^\prime\}_{r=1}^R$ from the ground truth posterior;
    \item \textbf{Posterior concentration ratio} ($\mathcal{R}$), computed as: 
    \[
    \sqrt{ \sum_r \| \theta_r - \overline{\theta}^\prime  \|_2^2  / \sum_r \| \theta_r^\prime - \overline{\theta}^\prime \|_2^2},
    \]
    comparing the concentration of the two posteriors around the ground truth mean (values close to one are desirable);
    \item \textbf{KL divergence}; ($\mathrm{D}_\text{KL}$) between a multivariate normal approximation of the aggregated posterior and a multivariante normal approximation of the true one, both computed from samples.
\end{itemize}
Experiments were repeated ten times for each value of $p$, in each of which a new $\theta^\prime$ was drawn. 
Additionally, average computing times for each aggregation method are shown in Table \ref{tab:log_time}.

When compared to the other methods, for all values of $p$, NAP presents a mean closer to the one obtained using centralized inference (smaller RMSE) and has a more accurate spread around it ($\mathcal{R}$ closer to one). In terms of KL divergence, only at $p=25$, PARAM outperforms NAP by a relatively small margin. Besides this case, NAP performs orders of magnitude better than the other methods, with increasing disparity as $p$ grows.

\begin{table*}[htb!]
    \centering
    \begin{tabular}{@{\extracolsep{4pt}}lccccccccccccc}
         & \multicolumn{3}{c}{$p=25$}  &  \multicolumn{3}{c}{$p=50$}  &  \multicolumn{3}{c}{$p=100$} \\ \cline{2-4} \cline{5-7} \cline{8-10}
        &  RMSE & $\mathcal{R}$ & $\mathrm{D}_\text{KL}$ & RMSE & $\mathcal{R}$ & $\mathrm{D}_\text{KL}$ & RMSE & $\mathcal{R}$ & $\mathrm{D}_\text{KL}$  \\
         NAP & $\bm{1.95}$ & $\bm{13.95}$ & 791.86 & $\bm{1.07}$ & $\bm{13.08}$ & $\bm{1539.32}$ &  $\bm{0.63}$ & $\bm{12.43}$ & $\bm{3493.35}$\\ 
         PART & 3.29 & 24.38 & 4263.53 & 2.44 & 31.27 & 20159.64 & 1.51 & 31.80 & 75423.10\\
         PARAM & 2.56 & 18.34 & $\bm{589.12}$ & 1.99 & 24.37 & 2568.57 & 1.32 & 26.07 & 11245.58 \\
         SP & 2.43 & 17.36 & 1586.80 & 2.02 & 24.62 & 7589.26 & 1.39 & 27.26 & 36994.45\\
         NP & 2.39 & 17.07 & 1343.88 & 2.01 & 24.54 & 7202.50 & 1.39 & 27.26 & 35313.77\\
         CON & 3.51 & 25.43 & 10654.78 & 3.08 & 37.92 & 56001.25 & 2.02 & 39.96 & 186275.86 \\
    \end{tabular}
    \caption{Comparison of different aggregation methods for embarassingly parallel MCMC inference on the logistic regression model with $p$ covariates. The values presented are averages over ten repetitions of the experiments. The best results are in bold.}
    \label{tab:log_res}
\end{table*}

\begin{table}[h!]
    \centering
    \begin{tabular}{@{\extracolsep{4pt}}lccc}
         & {$p=25$}  & {$p=50$}  &  {$p=100$} \\ \cline{2-4}
         NAP & 337.28 & 363.88 & 426.95\\ 
         PART & 117.10 & 245.85 &  727.99 \\
         PARAM & 33.36 & 62.45 & 115.40 \\
         SP & 476.90 & 749.07 & 18378.862 \\
         NP & 43.47 & 72.28 & 127.34 \\
         CON & 32.59 & 62.01 & 125.49 \\
    \end{tabular}
    \caption{Average computing times for different aggregation methods for the logistic regression experiment.}
    \label{tab:log_time}
\end{table}


\subsection{RARE CATEGORICAL EVENTS}

\begin{figure*}[ht!]
\centering
\includegraphics[width=1.0\textwidth]{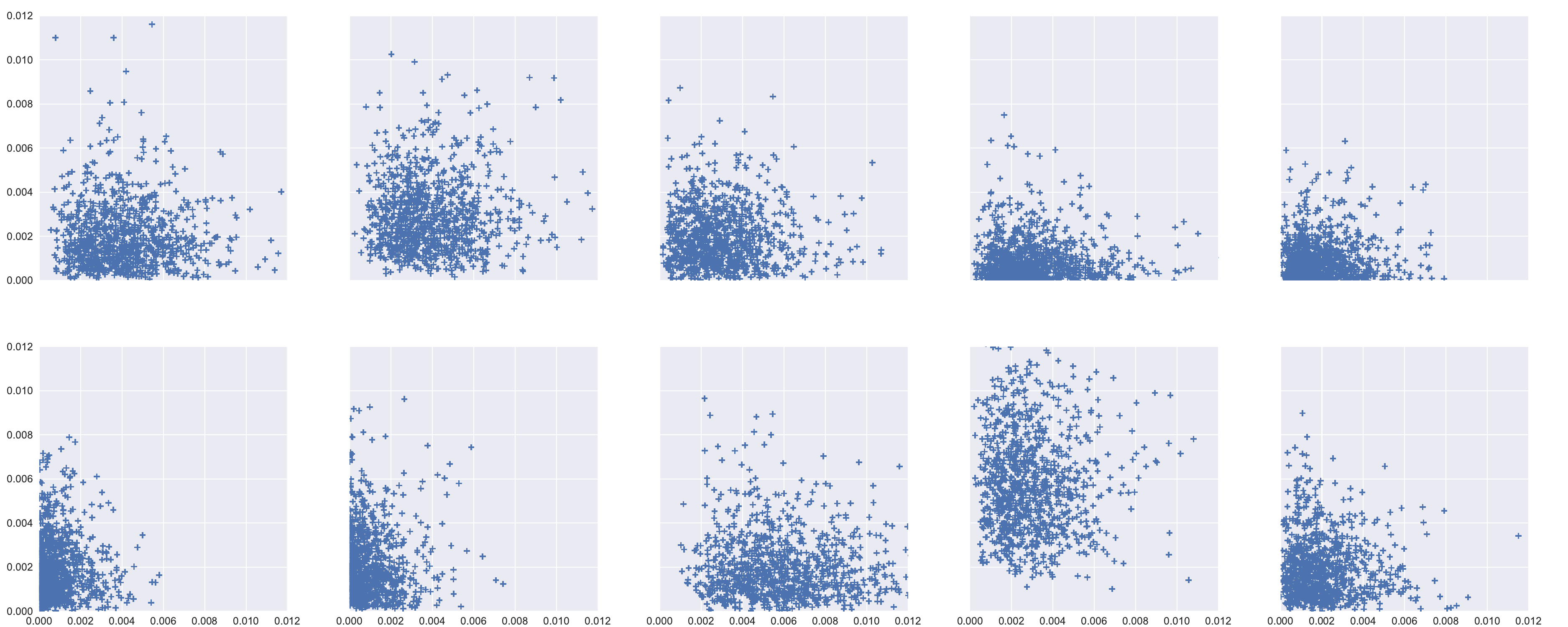}
\caption{Scatter plots of the marginal for $(r_1, r_2)$ for each of the $K=10$ subposteriors within one of the experiment rounds.}
\label{fig:cat}
\end{figure*}

In the scenarios explored in the previous experiments, there is no specific reason to believe that the subposteriors differ drastically from each other. We now consider parallel inference on the parameters $(\lambda_1, \lambda_2, \lambda_3)$ of a categorical model, with respective outcomes $A_1, A_2$ and $A_3$,  where the probability of observing one outcome is much higher than the others, i.e. $\lambda_3 \gg \lambda_1, \lambda_2$. 

We simulate $N=10000$ data points from $\mathrm{Categorical}(\lambda^\prime_1, \lambda^\prime_2, \lambda^\prime_3)$ with $\lambda_1^\prime = \lambda_2^\prime = 2 K / N$. We then partition the data into $K=10$ disjoint subsets, run MCMC in parallel and apply different aggregation methods to obtain samples from the approximate posterior. The aggregated posteriors are then compared with the one obtained using the complete data.

Since the expected number of $A_1$ and $A_2$ per partition is 2, it often occurs than some partitions have only $A_3$.  Figure \ref{fig:cat} illustrates how disparate the subposteriors can be, depending on the specific partitioning of data.

To compensate for the variability in experimental results due to the random partitioning of the subsets, we repeated the experiments one hundred times with different random seeds, and report average results in Table \ref{tab:cat}.  NAP clearly outperforms its competitors, with results that are orders of magnitude better.

\begin{table}[h!]
    \centering
    \resizebox{\columnwidth}{!}{
    \begin{tabular}{@{\extracolsep{4pt}}lccc}
        &  RMSE & $\mathcal{R}$ & $\mathrm{D}_\text{KL}$ \\ \cline{2-4}
         NAP & $\bm{0.71 \times 10^{-3}}$  & $\bm{2.61}$ & $\bm{106010.45 \times 10^{0}}$  \\ 
         PART & $0.27 \times 10^{-2}$ & 179.39  & $623035.37 \times 10^{1}$  \\
         PARAM & $0.11 \times 10^{-1}$ & 39.73 &  $469483.18 \times 10^{2}$ \\
         SP & $0.51 \times 10^{-2}$ & 267.30 & $945558.50 \times 10^{4}$ \\
         NP & $0.19 \times 10^{-1}$ & 268.97 & $968806.96 \times 10^{4}$ \\
         CON & $0.16 \times 10^{0}$ & 550.67   & $276976.28 \times 10^{3}$ \\
    \end{tabular}}
    \caption{Comparison of different aggregation methods for embarassingly parallel MCMC inference on the rare categorical events model. The best results are in bold.}
    \label{tab:cat}
\end{table}


\section{DISCUSSION}

We proposed an embarrassingly parallel MCMC scheme in which each subposterior density is mapped to a tractable form using a deep invertible generative model. 
We capitalized on the ease of sampling from the mapped subposteriors and evaluating their log density values to build an efficient importance sampling scheme to merge the subposteriors.
Imposing mild assumptions on the structure of the network, we proved that our importance sampling scheme is stable.

While in this work we gave special attention to the use of real NVP networks, our approach could potentially employ other invertible models, such as the Glow transform \citep{GLOW} or FFJORD \citep{grathwohl2018scalable}, without losing theoretical properties, as long as one can guarantee log densities remain bounded. 
If the bounds are difficult to verify, one could still resort to truncated forms of importance sampling \citep{Ionides2008, Vehtari2015} to control the variance of importance sampling estimates.

Our experimental results demonstrated that NAP is capable of capturing intricate posteriors and coping with heteregenous subposteriors.
In particular, we observed that it significantly outperformed current methods in high-dimensional settings. A possible explanation for this is that, unlike the density estimation techniques underlying the competing methods, the real NVP transformations used in our method, are specifically designed for high-dimensional data such as images. 

Finally, the generative models we use serve as a intermediate representation to the subposterior, the size of  which does not depend on the number of subposterior samples. 
Thus, workers can produce arbitrarily accurate subposterior estimates by drawing additional samples, without 
affecting the cost of communicating the subposteriors to the server, or the computational cost of aggregating them into a final posterior estimate.

\label{sec:discussion}

\subsubsection*{Acknowledgements}
DM, PB and SK were funded by the Academy of Finland, grants 319264 and 294238. The authors gratefully acknowledge the computational resources provided by the Aalto Science-IT project and support from the Finnish Center for Artificial Intelligence (FCAI).


%

\bibliographystyle{plainnat}
\bibliography{references}

\end{document}